\documentclass[letterpaper, 10 pt, conference]{ieeeconf}  

\IEEEoverridecommandlockouts                              

\usepackage[ruled, noend, linesnumbered]{algorithm2e}
\usepackage{amsmath}
\usepackage{amssymb}
\usepackage{dsfont}
\usepackage{bm}
\usepackage{tabularx, booktabs}
\usepackage{multicol}
\usepackage{multirow}
\usepackage{dblfloatfix} 
\usepackage{graphicx}
\usepackage{subfig}

\newcommand{\ra}[1]{\renewcommand{\arraystretch}{#1}}

\DeclareMathOperator*{\argmin}{argmin}
\DeclareMathOperator*{\argmax}{argmax}

\renewcommand{\vec}[1]{\bm{#1}}

\usepackage{amsthm}
\newtheorem{claim}{Claim}

\title{\LARGE \bf
Forward Kinematics Kernel for Improved Proxy Collision Checking
}

\author{Nikhil Das and Michael C. Yip, \IEEEmembership{Member, IEEE}
\thanks{Nikhil Das and Michael C. Yip are with the Department of Electrical and Computer Engineering, University of California San Diego
        {\tt\small \{nrdas,yip\}@ucsd.edu}}%
}

\begin{document}

\maketitle
\thispagestyle{empty}
\pagestyle{empty}

\begin{abstract}
Kernel functions may be used in robotics for comparing different poses of a robot, such as in collision checking, inverse kinematics, and motion planning. These comparisons provide distance metrics often based on joint space measurements and are performed hundreds or thousands of times a second, continuously for changing environments. Few examples exist in creating new kernels, despite their significant effect on computational performance and robustness in robot control and planning.  We  introduce  a  new  kernel  function  based  on forward  kinematics  (FK)  to compare  robot  manipulator configurations. We integrate our new FK kernel  into  our proxy collision checker, Fastron, that previously showed significant speed improvements to collision checking and motion planning. With the new FK kernel, we realize  a two-fold speedup  in proxy collision check speed, 8 times less memory, and a boost in classification accuracy from 75\% to over 95\% for a 7 degrees-of-freedom robot arm compared to the previously-used radial basis function kernel.  Compared  to  state-of-the-art  geometric collision  checkers,  with the FK kernel, collision checks are now 9 times faster. To show the broadness of the approach, we apply Fastron FK in OMPL across a wide variety of motion planners, showing unanimously faster robot planning.
\end{abstract}

\section{Introduction}
Motion planning for robotics is the task of finding a sequence of feasible robot states between a start and goal state. Feasible robot states are those that satisfy problem-specific constraints, generally one of which is to remain out of collision from objects in the environment. Motion planning for a robot is often performed in its configuration space (C-space), a space in which each element represents a unique configuration of the robot. Each configuration in the robot's C-space can belong to one of two subspaces: $\mathcal{C}_{free}$ or $\mathcal{C}_{obs}$. A configuration is in the collision-free $\mathcal{C}_{free}$ subspace if the robot is not in contact with any workspace obstacle when in that corresponding configuration; otherwise, the configuration is in the in-collision $\mathcal{C}_{obs}$ subspace \cite{Choset2005, Lozano1983}. For robot manipulators with many controllable degrees-of-freedom (DOF), the dimensionality of the C-space may be large as each joint contributes to defining the robot's configuration. An example workspace and its corresponding C-space for a 3 DOF robot arm are shown in Fig. \ref{fig:workspace}-\ref{fig:gtCspace}.

Determining feasibility of a robot state usually involves discrete collision checking to determine whether the query robot state would intersect with an environmental obstacle \cite{Pan2015}. Repeated collision checking is computationally expensive and takes up to 90\% of the computation time during motion planning \cite{Elbanhawi2014}. An open challenge exists for real-time, accurate, dynamic robot planning in changing environments; to meet this need, we seek to create an accurate C-space model that may be used as a faster proxy to the more computationally-burdensome geometric collision checkers upon which most motion planners currently rely.

\begin{figure}[t]
	\centering
    \subfloat[Workspace\label{fig:workspace}] {
        \includegraphics[width=0.48\linewidth, trim={2cm 2cm 1cm 0.5cm}, clip]{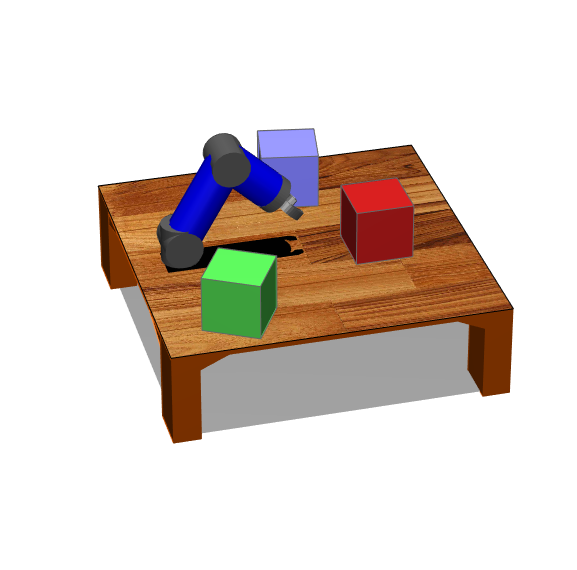}}
		\hfill
    \subfloat[Ground truth\label{fig:gtCspace}] {
        \includegraphics[width=0.48\linewidth, trim={0cm 1cm 0cm 1.5cm}, clip]{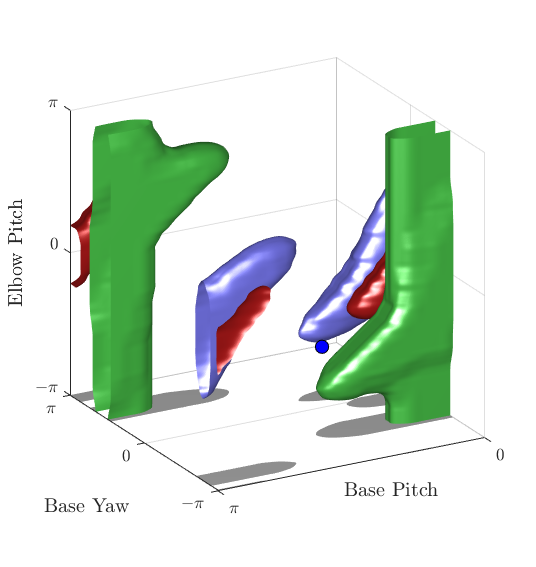}}
	\hfill
	\subfloat[Fastron RBF\label{fig:rbfCspace}] {
        \includegraphics[width=0.48\linewidth, trim={0cm 1cm 0cm 1.5cm}, clip]{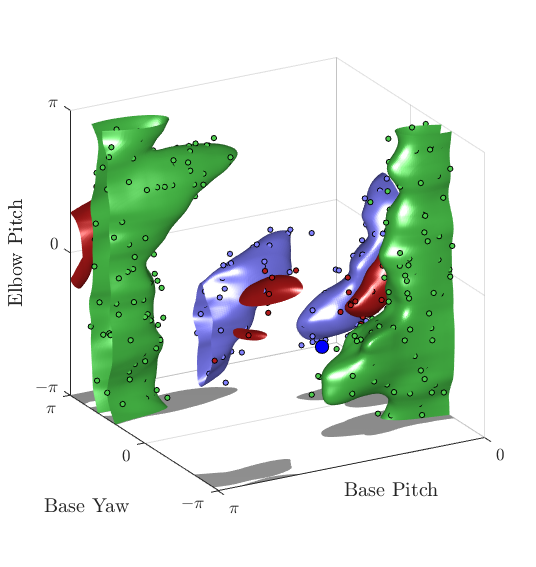}}
		\hfill
	\subfloat[Fastron FK\label{fig:fkCspace}] {
        \includegraphics[width=0.48\linewidth, trim={0cm 1cm 0cm 1.5cm}, clip]{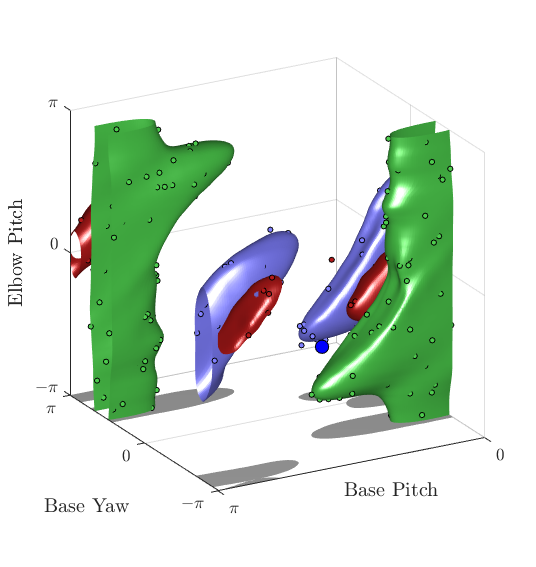}}
    \caption{(a) A 3 DOF robot and multiple cube obstacles. (b) The ground truth C-space model determined by calling a geometric collision detector on a grid of samples. (c) A Fastron model using the RBF kernel from prior work. (d) A Fastron model using the proposed FK kernel. Both versions of the Fastron models are trained on the same data, and the resulting support points for each model are shown. The Fastron FK model closely matches the ground truth model while the Fastron RBF model is uneven and requires many more support points.}
    \label{fig:fastronCompare}
\end{figure}

\subsection{Contributions}
In this paper, we present a new kernel function, or similarity function, for improved configuration space modeling for robot manipulators. We show improved proxy collision checking performance in terms of both collision check times and accuracy with the new kernel, and we demonstrate the kernel's usefulness in accelerating motion planning.

In previous work, we described the Fastron algorithm, which creates a nonparametric model inspired by the kernel perceptron algorithm and efficiently checks for changes in the C-space due to moving workspace obstacles \cite{Das2019}. The input into the Fastron model is a robot configuration, and the output is its predicted collision status. In previous work, a radial basis function (RBF) was used for the kernel function, which isotropically compares the joint values of a robot manipulator. While the RBF kernel is an adequate general-purpose kernel function, it fails to capture the influence each joint has in the manipulator's workspace. In this work, we replace the RBF kernel function with a new kernel function we call the forward kinematics (FK) kernel. The benefits of replacing the RBF kernel with the FK kernel may be observed in Fig. \ref{fig:rbfCspace}-\ref{fig:fkCspace}, where the Fastron model closely resembles the ground truth with the FK kernel but is uneven and requires many more points to represent the model when using the original RBF kernel.

We show that the FK kernel is positive definite and prove our modeling algorithm can always find a set of weights that correctly predicts the collision status of a training set. We also show that when using the Fastron model as a proxy collision checker, using the FK kernel can provide collision status predictions twice as fast as when using the RBF kernel and requires approximately an order of magnitude fewer support points to represent the model. Using the FK kernel improves the accuracy of collision status prediction to be over 95\% for a 7 DOF robot arm, a significant improvement over the 75\% accuracy achieved with the RBF kernel. We apply the Fastron algorithm with the FK kernel to various motion planners, demonstrating up to 3 times faster motion plans compared to when using other collision checkers.

\begin{figure*}[b]
	\centering
	\subfloat[Example workspace\label{fig:workspaceComparison}]{%
       \includegraphics[width=0.33\linewidth, trim={1cm 0.5cm 0.5cm 1cm}, clip]{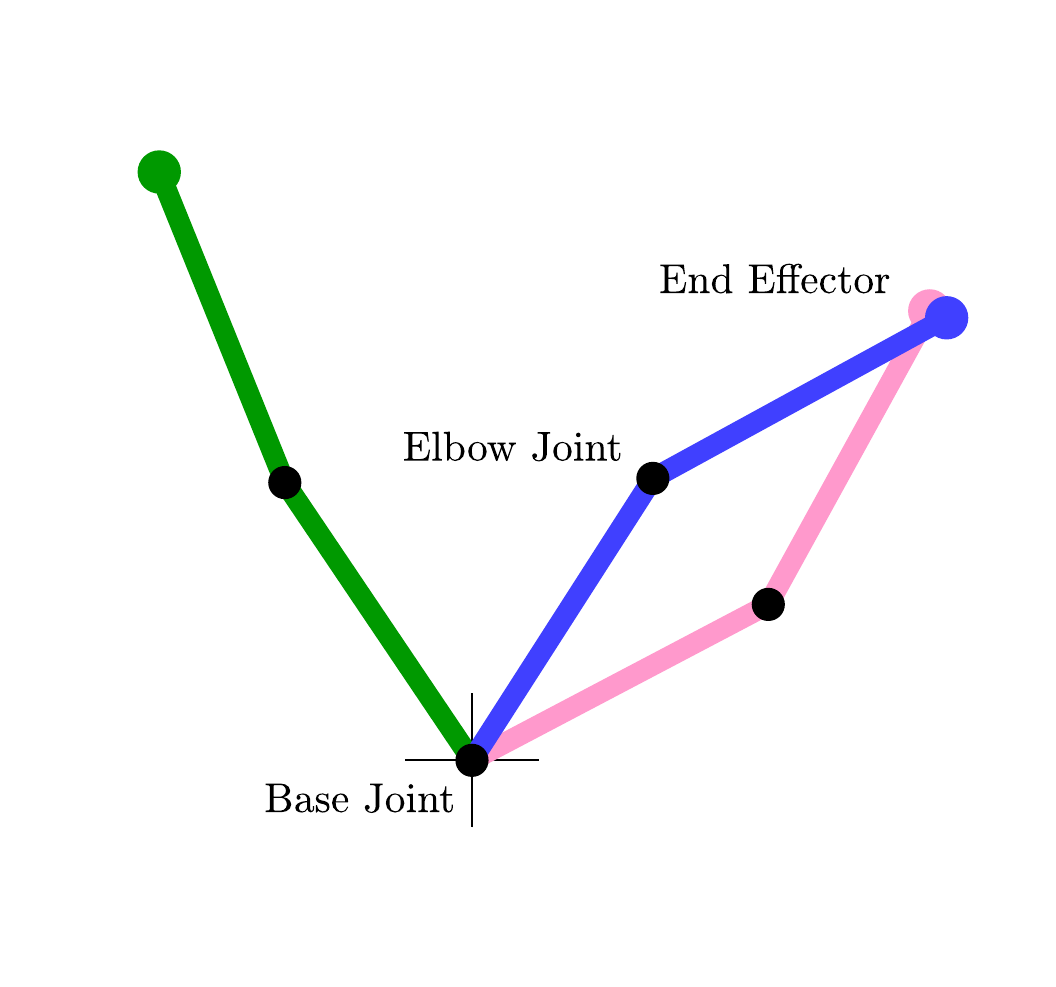}}
    \hfill
    \subfloat[RBF kernel in configuration space\label{fig:rbfKernelComparison}]{%
        \includegraphics[width=0.33\linewidth, trim={0cm 0cm 1cm 0cm}, clip]{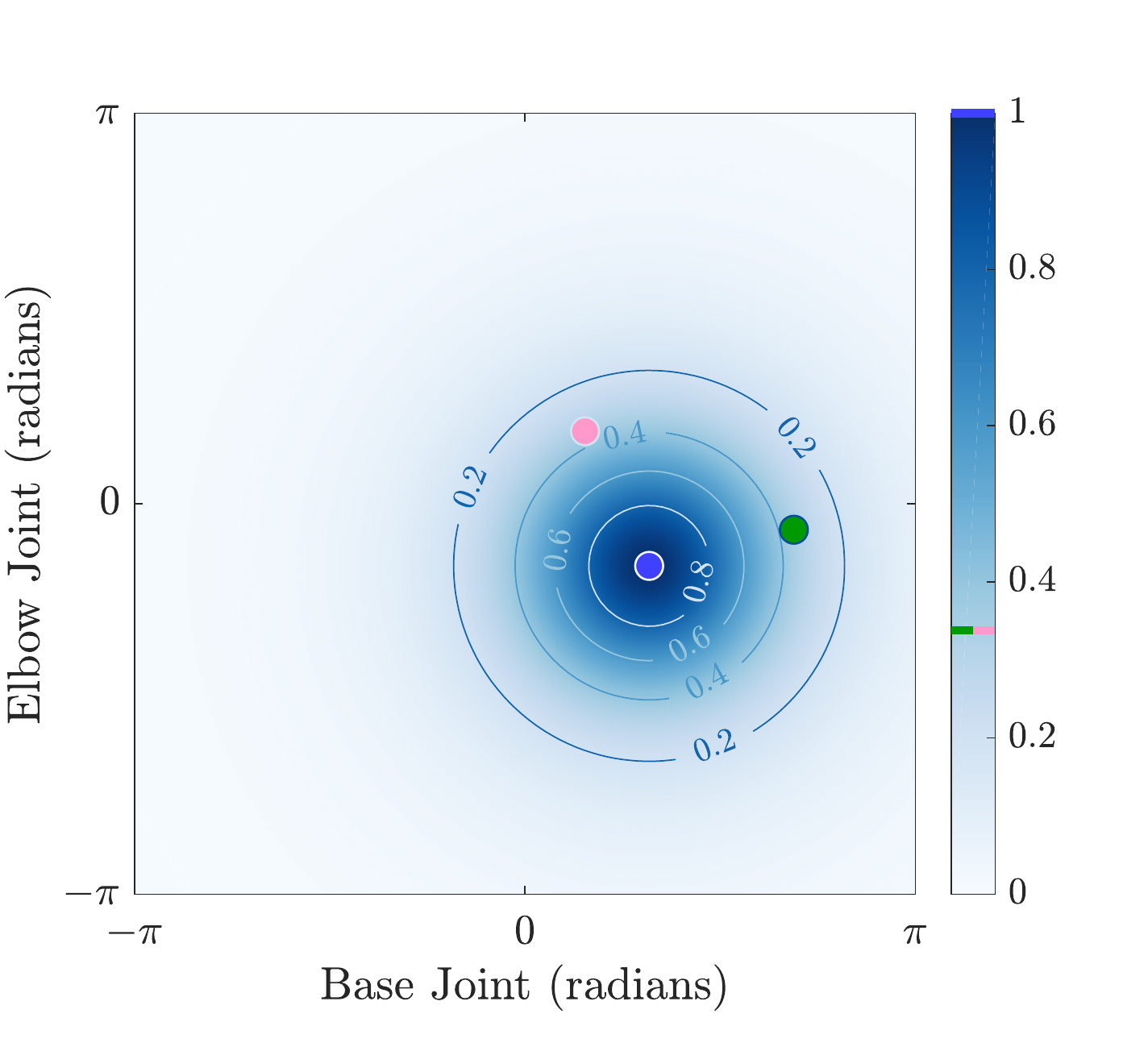}}
        \hfill
    \subfloat[FK kernel in configuration space\label{fig:fkKernelComparison}]{%
        \includegraphics[width=0.33\linewidth, trim={0cm 0cm 1cm 0cm}, clip]{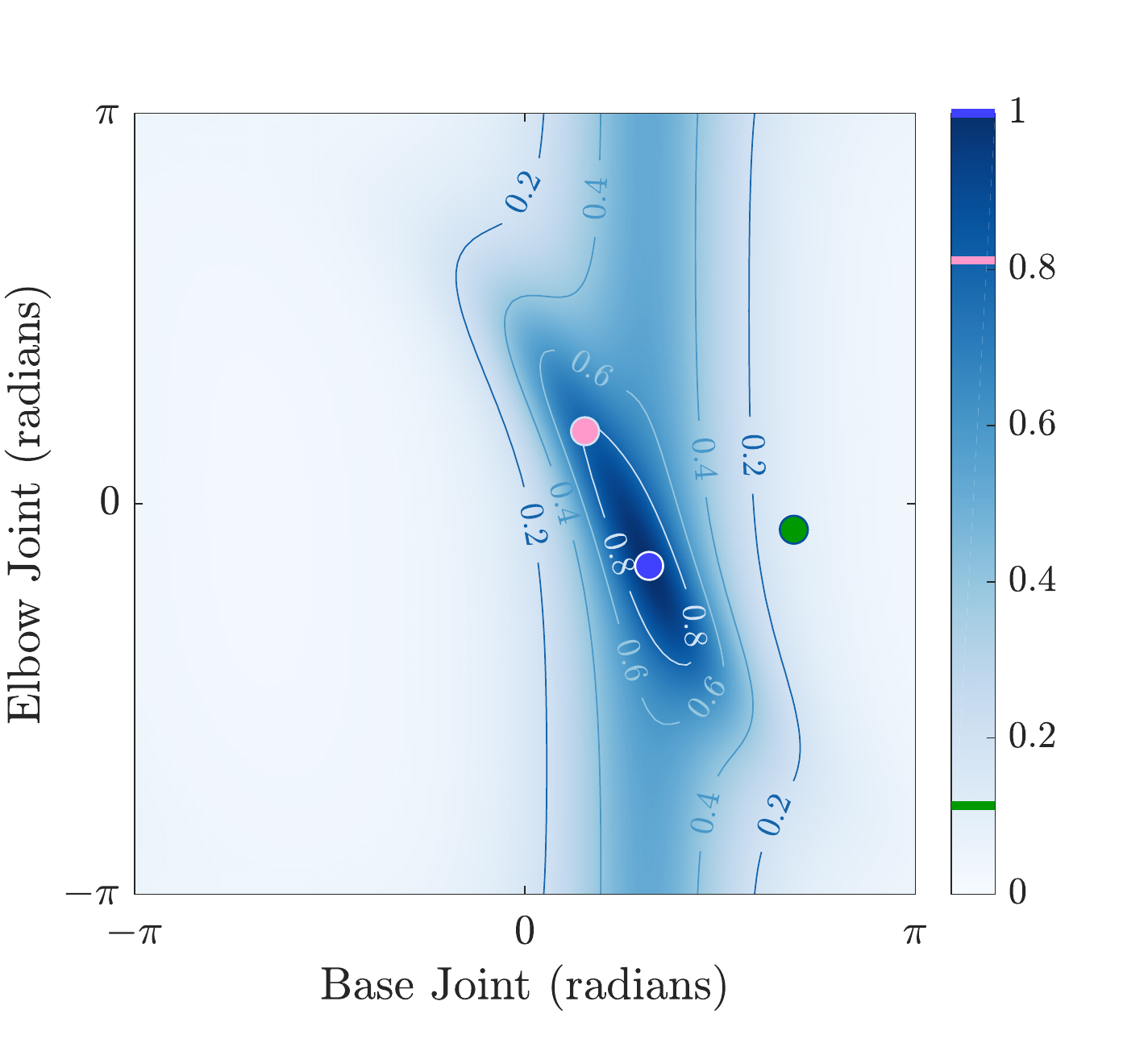}}
    \caption{(a) A workspace representation of one configuration of a two-link arm shown in blue along with two alternate configurations of the same robot shown in green and pink. (b) A visualization of the RBF kernel in C-space centered on the blue configuration. According to the RBF kernel, the blue configuration is as similar to the pink configuration as it is the green configuration, as they reside on the same isocontour. (c) A visualization of the FK kernel in C-space centered on the blue configuration. According to the FK kernel, the blue configuration is more similar to the pink configuration than to the green, which is more representative of the proximities of the blue and pink configurations in the workspace representation.}
    \label{fig:kernelComparison}
\end{figure*}

\subsection{Related Work}
The following works describe proxy checking methods that try to avoid computationally-burdensome collision checking and methods that utilize control points (which are used in the proposed FK kernel) for other kinematic contexts.

\subsubsection{Configuration Space Modeling}
Various machine learning-based techniques have been developed to model C-space or for proxy collision checking, including support vector machine models (SVMs) \cite{Pan2015}, Gaussian mixture models (GMMs) \cite{Huh2016}, k-nearest neighbor models (KNNs) \cite{Pan2016}, and kernel perceptrons \cite{Das2019}. Most of these methods provide either high accuracy but very poor computational speed, or high computational speed but poor accuracy and limited capacity to model changing environments. 

Pan et al.'s SVM method learns a C-space model for a pair of rigid-body objects \cite{Pan2015}. Their method provides near-perfect models because they employ an active learning strategy that iteratively improves the precision of the model, and they use their SVM-based model for distance/penetration depth estimation between two bodies using a constrained nonlinear optimizer. However, to apply this method to a kinematic chain, a new model must be learned offline for each pair of objects, including each link in the kinematic chain and each workspace obstacle. 

In \cite{Huh2016}, active learning of a GMM model is associated with sample-based motion planners itself by using the samples from a rapidly-exploring random tree (RRT) motion planner to modify the GMM. GMMs are constrained by the fixed dimensionality of the model.


KNNs have also been used for faster proxy collision detection for static environments \cite{Pan2016}. Locality-sensitive hashing is utilized for extremely fast query times. However, the problem of rehashing given changing environments makes this unsuitable for changing environments.

Our Fastron algorithm uses a kernel perceptron model to globally represent the C-space and uses an active learning strategy to adapt to a changing workspace. It was shown that collision checks may be an order of magnitude faster than state-of-the-art collision detectors \cite{Das2019}. However, the accuracy of the model must be improved to reduce the search time for valid configurations and to reduce the amount of time spent repairing infeasible parts of a motion plan.

While each of these learning-based methods generates an approximate C-space model, the workspace and C-space are treated as completely separate entities and the link that relates the two spaces is more or less neglected. In this work, we aim to link the two spaces via FK.

\subsubsection{Control Points for Collision-Free Motion}

Control points are a set of hand-selected points placed on a robot and their locations may be used to indicate the position of a rigid-body robot or a kinematic chain manipulator. Control points should be chosen to constrain the configuration of the robot, and the minimum number of these points depends on the number of DOF of the robot \cite{Choset2005}. FK provides the positions of these control points for manipulators given a set of joint values, and we consider FK to be the link that relates a robot's position in C-space to its position in the workspace. 

Artificial potential fields can be used for motion planning in a robot's C-space, where the sum of attractive and repulsive forces drive the robot away from C-space obstacles and toward a goal \cite{Khatib1985}. Potential fields can be applied to kinematic chain manipulators using control points placed along the kinematic chain \cite{Choset2005}.

RelaxedIK, a motion-synthesis method for manipulators, uses a neural network to estimate the risk of self-collision for an input robot configuration \cite{Rakita2018}. Control points are placed at each joint and their locations in the workspace are included in the input to the network. The use of control points allows estimation of the self-collision costs two orders of magnitude faster than a distance-based method, and optimizing over these self-collision costs allows finding configurations that are free from self-collisions.

In this work, we utilize control points in creating a similarity function between manipulator configurations, and we leverage how FK relates a manipulator's C-space to its workspace for more accurate proxy collision checking.

\section{Background}
\subsection{Forward Kinematics}

FK for a kinematic chain determines the pose (position and orientation) of a point on the chain given the joint values. When using the standard Denavit-Hartenberg (D-H) convention, the homogenous transformation matrix of the pose of joint $i$ in the frame of joint $i-1$ is
\begin{equation}
	\vec{^{i-1}\!A_i} = 
	\begin{bmatrix}
		\mathrm c(\theta_i) & -\mathrm c(\alpha_i)\mathrm s(\theta_i) & \mathrm s(\alpha_i)\mathrm s(\theta_i) & a_i\mathrm c(\theta_i) \\
		\mathrm s(\theta_i) & \mathrm c(\alpha_i)\mathrm c(\theta_i) & -\mathrm s(\alpha_i)\mathrm c(\theta_i) & a_i\mathrm s(\theta_i) \\
		0 & \mathrm s(\alpha_i) & \mathrm c(\alpha_i) & d_i \\
		0 & 0 & 0 & 1 \\
	\end{bmatrix}
\end{equation}
where $d_i$ and $a_i$ are the distances between the joints' origins along the $i^{th}$ $x$-axis and $z$-axis, respectively, $\theta_i$ and $\alpha_i$ are the angles between the joints' $x$-axes measured along the $i^{th}$ $z$-axis and vice versa, respectively, and $\mathrm{c}(\cdot)$ and $\mathrm{s}(\cdot)$ represent $\cos(\cdot)$ and $\sin(\cdot)$, respectively. $\theta_i$ and $d_i$ are the joint variables for revolute and prismatic joints, respectively, and a configuration of a robot with $D$ DOF is represented as a $D$-dimensional vector of these joint variables. The other D-H parameters are fixed depending on the geometry of the robot. Note that when using the standard D-H convention, the origin of the frame of the $i^{th}$ joint is placed at the distal end of the $i^{th}$ link. For more details, see \cite{Rocha2011}.

Starting with the base transform relative to a world frame $\vec{^w\!A_0}$, the pose of the $i^{th}$ joint is $\vec{^w\!A_i} = \!\vec{^w\!A_0}\!\vec{^0\!A_1}\!\dots{\vec{^{i-1}\!A_i}}$, where each transform in the chain is dependent on the robot configuration as described in the previous paragraph. The pose for an arbitrary point on the $i^{th}$ link is found by appending an additional static transformation $\vec T$ to the pose of the $i^{th}$ joint: $\vec{^w\!A_iT}$.

Let $\mathrm{pos}(\cdot)$ extract the position from a homogeneous transformation matrix, i.e., the first three elements of the last column. We can thus represent the position of the $m^{th}$ control point for configuration $\vec x \in  \mathbb{R}^D$ as
\begin{equation}
FK_m(\vec x) =\mathrm{pos}\left(\vec{^w\!A_{j_m}}\vec{T_m}\right)
\end{equation}
where $\vec{T_m}$ is the static transform applied to the the pose of the joint with index $j_m$.

\subsection{Binary Classification and Kernel Functions}
\label{sec:kernelFunctions}
A binary classifier predicts the class label of a query point. A binary classifier model is fit to a dataset $\mathcal{X} = \{\mathcal{X}_1, \ldots, \mathcal{X}_N\}$ and its associated vector of labels $\vec y$. Each $\mathcal{X}_i\in \mathbb{R}^D$ represents a configuration of a robot. We let $\vec{y}_i = +1$ denote $\mathcal{X}_i\in\mathcal{C}_{obs}$ and $\vec{y}_i = -1$ denote $\mathcal{X}_i\in\mathcal{C}_{free}$.

Binary classification models often use kernels to describe nonlinear seperation of classes. A kernel function $K:\mathbb{R}^D \times \mathbb{R}^D \rightarrow \mathbb{R}$ is a similarity function that should provide a large score for similar configurations and a low score for dissimilar configurations. RBF kernels, which are functions of the distance between two points, such as the Gaussian kernel, are popular. In this paper, we use the second-order rational quadratic kernel \cite{Rasmussen2006} for all RBF kernels, defined as
\begin{equation}
	K_{RQ}(\vec{x},\vec{x}') = \left(1+\frac{\gamma}{2}\|\vec{x}-\vec{x}'\|^2\right)^{-2} \label{eq:rbf}
\end{equation}
where $\gamma > 0$ is a parameter dictating the width of the kernel.

Each element of a kernel matrix $\vec{K}$ is defined elementwise for each pair of points in $\mathcal{X}$: $\vec{K}_{ij} = K(\mathcal{X}_i,\mathcal{X}_j)$. $\vec K$ is called positive definite (PD) if $\vec{\alpha}^\mathsf{T}\vec{K}\vec{\alpha} \geq 0\ \forall \vec \alpha \in \mathbb{R}^N$, and equality occurs only for $\vec\alpha = \vec{0}$. $\vec{K}$ is considered to be positive semidefinite (PSD) if equality may occur for $\vec\alpha\neq \vec{0}$.


A PD kernel is one that yields a PD kernel matrix $\vec K$ when each configuration in set $\mathcal{X}$ is unique \cite{Hofmann2008}. Positive definiteness is an important property for a kernel function because it guarantees nonsingularity for the kernel matrix and represents an implicit mapping to a higher dimensional feature space according to Mercer's theorem \cite{Minh}. PSD kernels and matrices do not have the same guarantees their PD counterparts have. $K_{RQ}(\cdot, \cdot)$ is an example of a PD kernel \cite{Rasmussen2006}. Summing PD kernels or multiplying PD kernels with a scalar also yields a PD kernel \cite{Hofmann2008}.



\section{Methods}
\subsection{Forward Kinematics Kernel}
\subsubsection{Definition}
The major insight of this paper is in a more intuitive distance metric by computing distances in the workspace rather than in the joint space. Computing distances directly between joint configurations does not take into consideration the robot's position in the workspace. Consider the blue two-link arm shown in Fig. \ref{fig:workspaceComparison}. Comparing the pink and green configurations to the blue, the pink configuration appears much closer to the blue than the green does. However, the distance in the C-space between the blue and pink configurations (represented as colored dots) is equal to that of the blue and green configurations as can be seen in Fig. \ref{fig:rbfKernelComparison} and \ref{fig:fkKernelComparison}. Consequently, according to the RBF kernel, the blue configuration is as similar to the pink configuration as it is the green configuration as shown in Fig. \ref{fig:rbfKernelComparison}.


Rather than comparing the joint values of two configurations directly, we use FK and compare distances in the workspace between control points on the arm. We can thus define the FK kernel as a sum of RBF kernels evaluated between each corresponding control point location:
\begin{equation}
K_{FK}(\vec{x}, \vec{x}') = \frac{1}{M}\sum_{m=1}^MK_{RQ}\left(FK_m(\vec{x}), FK_m(\vec{x}')\right) \label{eq:fkKernel}
\end{equation}
where $FK_m(\cdot)$ provides the position of the $m^{th}$ control point and $M$ is the number of control points representing the kinematic chain. 

In this paper, we only place control points at joint frame origins (using the standard D-H convention) by setting $\vec{T_m} = \vec I$ for each control point. In other words, a control point is placed at the distal end of each link. However, control points may be placed at arbitrary locations along the manipulator if desired. If the origins of the frames of two subsequent joints in the kinematic chain are coincident, we only place one control point at the overlapping origin. In other words, using the standard D-H convention, if a revolute joint has both $d_i$ and $a_i$ equal to 0, a control point is not placed for this joint as its origin overlaps with the preceding joint origin.

Fig. \ref{fig:fkKernelComparison} provides a visualization of the FK kernel for the blue two-link arm in Fig. \ref{fig:workspaceComparison}, where control points are placed at the location labeled ``Elbow Joint'' and ``End Effector''. According to the FK kernel, the similarity score between the blue and pink configurations is about 0.8, while the similarity between the blue and green is about 0.1. The similarity scores produced by the FK kernel better represent the differences we perceive in the positions of the robots in the workspace.



\subsubsection{Positive Definiteness}
As previously mentioned, a PD kernel guarantees a PD and nonsingular kernel matrix and represents an implicit mapping to a feature space. The RBF kernel defined in Eq. \ref{eq:rbf} evaluated on dataset $\mathcal{X}$ will give rise to a PD kernel matrix as long as each configuration in $\mathcal{X}$ is unique. We show below this is also true for the FK kernel. 


The FK kernel matrix based on the FK kernel defined in Eq. \ref{eq:fkKernel} is a sum of RBF kernel matrices evaluated on $M$ control points: $\vec K = \frac{1}{M}\sum_i^M\vec{K_m}$. We henceforth refer to the RBF kernel matrices that comprise the FK kernel matrix as \textit{summand kernel matrices}. The $m^{th}$ summand kernel matrix $\vec{K_m}$ is PD as long as the locations of the $m^{th}$ control point (determined by applying $FK_m(\cdot)$ to the configurations in $\mathcal{X}$) are unique. However, depending on the choices of control points or kinematic redundancy, it is possible for two unique configurations in $\mathcal{X}$ to yield the same control point location. For example, Fig. \ref{fig:psdKernelWorkspace} shows four unique configurations of a two-link arm, but there are only three unique locations for each control point (represented as a hollow square and circle on each configuration). With coincident control point locations, the summand kernel matrices are demoted to PSD.


Despite possibly having PSD summand kernels, the \textit{FK kernel matrix is PD} if the intersection of the nullspaces of the summand matrices is the zero vector.

\begin{figure}[!t]
	\centering
	\includegraphics[width=0.8\linewidth, trim={.3cm 0.6cm 0.cm 0.3cm}, clip]{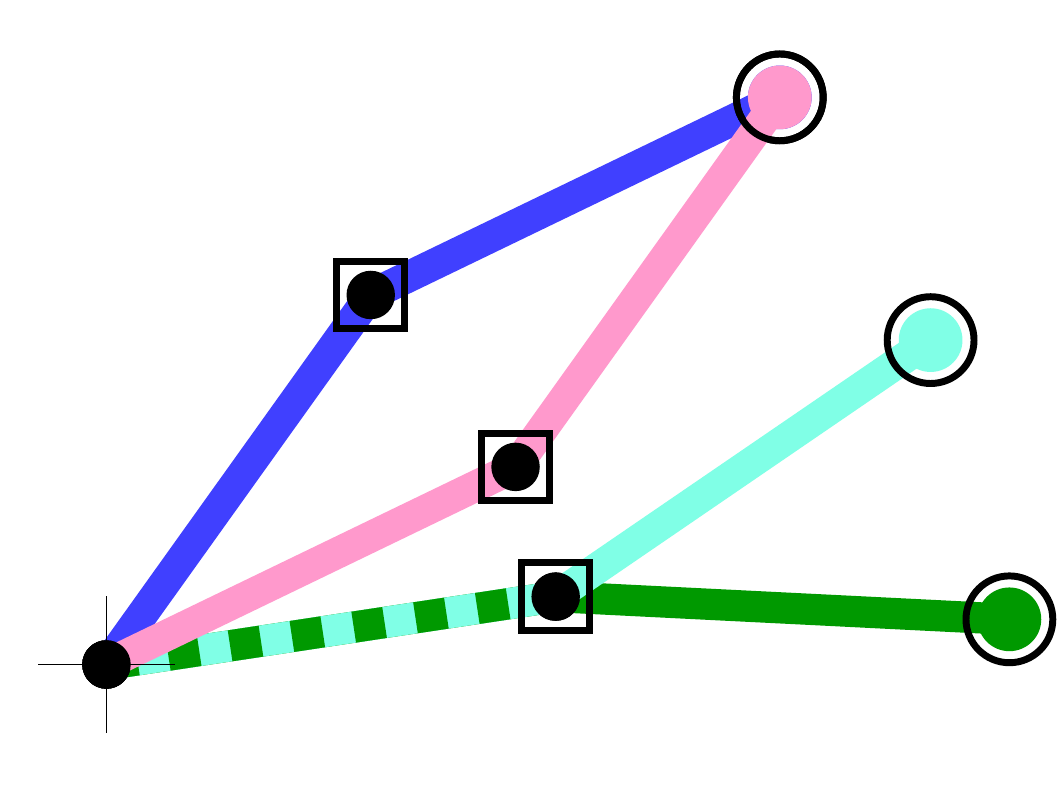}
	\caption{Four unique configurations of a two-link arm. There are two control points selected along the arm, marked by the square and hollow circle. Overlapping links are marked with a dashed line. Despite being unique configuations, some control points overlap, causing summand kernel matrices to be positive semidefinite. However, the FK kernel is positive definite as long as the configurations with coincident control points are distinguishable based on the other control points.}
	\label{fig:psdKernelWorkspace}
\end{figure}




%

\begin{claim}
The FK kernel matrix $\vec{K}=\frac{1}{M}\sum_{m=1}^M{\vec{K_m}}$ is PD if the nullspaces of each summand kernel matrix $\vec{K_m}$ share only the zero vector, $\bigcap_{m=1}^MN(\vec{K_m}) = \vec{0}$.
\label{claim:pdFKkernel}
\end{claim}

\begin{proof}
Assume $\vec v \in \bigcap_{m=1}^MN(\vec{K_m})$. Thus, $\vec{K_1}\vec{v} + \dots + \vec{K_M}\vec{v} = M\vec{Kv} = 0\Rightarrow \bigcap_{m=1}^MN(\vec{K_m}) \subseteq N(\vec{K})$.

For $\vec{K}$ to be PD, it must also be nonsingular, which means $N(\vec{K})=\vec 0$. Thus, $\bigcap_{m=1}^MN(\vec{K_m}) = \vec 0$.
\end{proof}

What Claim \ref{claim:pdFKkernel} means intuitively is that even if some control point locations coincide for two unique configurations, the FK kernel is PD if these configurations are distinguishable using the locations of the other control points. For example, in Fig. \ref{fig:psdKernelWorkspace}, the locations of the control points marked with hollow circles differentiate the configurations with coincident control points marked with squares, and vice versa. Note that in practice, if the configurations in $\mathcal{X}$ are generated via random sampling, control point locations should almost surely be unique.

\subsection{Fastron Model and Algorithm}
The Fastron algorithm (described in detail in \cite{Das2019}) generates a binary classification model that is used for proxy collision checking. The Fastron model is a weighted sum of kernel functions: $f(\vec x) = \sum_{i=1}^{N} K(\mathcal{X}_i, \vec x)\vec{\alpha}_i$, where $\vec \alpha\in \mathbb{R}^N$ is a vector of model weights, $\mathcal{X}$ is a set of $N$ configurations, and $K(\cdot, \cdot)$ is a PD kernel. Proxy collision detection is performed by determining the sign of $f(\vec x)$, where $f(\vec x)>0$ denotes the query configuration $\vec x$ is predicted to be in $\mathcal{C}_{obs}$ and $f(\vec x)<0$ denotes $\mathcal{C}_{free}$. $\vec x$ is considered to be correctly classified if its margin $m(\vec x) = yf(\vec x)$ is positive, where $y$ is its true label.

The weight vector $\vec\alpha$ is found by minimizing the loss
\begin{equation}
\mathcal{L}(\vec\alpha) = \frac{1}{2}\vec\alpha^\mathsf{T}\vec K \vec\alpha - \vec y^\mathsf{T} \vec B \vec\alpha
\end{equation}
subject to $\vec y_i\vec\alpha_i \geq 0\ \forall i$, where $\vec y$ is a vector of labels for dataset $\mathcal{X}$, $\vec K$ is the kernel matrix for $\mathcal{X}$, and $\vec B$ is a diagonal matrix where the $i^{th}$ diagonal element is $b_i=\beta^{0.5(\vec y_i + 1)}$ and $\beta\geq 1$ is a user-selected conditional bias parameter, which biases the model toward predicting $\mathcal{C}_{obs}$ for more conservative  predictions. This loss function is similar to that used in SVMs \cite{Cortes1995}. The Fastron algorithm utilizes greedy coordinate descent and exits as soon as the margin $m(\mathcal{X}_i) = \vec y_i\vec F_i > 0\ \forall \mathcal{X}_i\in \mathcal{X}$, where $\vec F_i = f(\mathcal{X}_i)$. Some advantages of greedy coordinate descent we realize are it promotes model sparsity (and thus computational efficiency), allows us to perform lazy kernel matrix evaluation (i.e., a column of $\vec K$ can be computed and stored only when needed), and allows us to define a cheap update rule (described below) \cite{Das2019}.

In what follows, parenthetical superscripts denote the training iteration upon which the given value depends. On iteration $n$, $\vec \alpha^{(n+1)}$ and $\vec F^{(n+1)}$ are determined incrementally:
\begin{align}
&\Delta \vec\alpha_i^{(n)} = b_i\vec y_i - \vec F_i^{(n)}\\
&\vec\alpha^{(n+1)} = \vec\alpha^{(n)} + \Delta \vec\alpha_i^{(n)}\vec{e(i)}  \label{updateRule} \\
&\vec F^{(n+1)} = \vec F^{(n)} + \Delta \vec\alpha_i^{(n)}\vec K\vec{e(i)}
\end{align}
where $\vec{e(i)}$ is the $i^{th}$ standard basis vector.
The descent direction, as stated previously, is determined greedily by selecting the point with the most negative margin:
\begin{equation}
	i = \argmin_j{\left(m^{(n)}(\mathcal{X}_j)\right)} \label{minMargin}
\end{equation}
This descent rule guarantees positive margin for all points in $\mathcal{X}$ in a finite number of iterations.



\SetAlgoSkip{medskip}
\begin{algorithm}[ht]
\small

\SetKwInput{Param}{Parameters}
\SetKw{logicalor}{OR}
\SetKw{continue}{continue}
\SetKw{break}{break}
\DontPrintSemicolon

 \KwIn{Dataset of configurations $\mathcal{X}$, collision status labels $\vec{y}$}
  \KwOut{Updated $\vec{\alpha}$, support set of configurations $\mathcal{S}$}
  \tcp{Get values from previous update}
  $\vec{\alpha}, \vec{F}, \vec{K} \leftarrow \mathrm{loadPreviousModel}()$ \\
 \For {$iter = 1$ to $iter_{max}$}
 {
   \tcp{Check for misclassifications}
 \uIf{$\min{\vec{y}\vec{F}} \leq 0$}{
 		$i \leftarrow \argmin \vec{y}\vec{F}$ \\
        $\mathrm{computeKernelMatrixColumn}(i)$ \\
        \tcp{Add/adjust support point}
        \uIf {$\mathrm{count}(\vec{\alpha}\neq 0)<\mathcal{S}_{max}$ \logicalor $\vec{\alpha}_i\neq 0$}{
  		  $\Delta\vec{\alpha} \leftarrow b_i\vec{y}_i-\vec{F}_i$ \\
          $\vec{\alpha}_i \leftarrow \vec{\alpha}_i + \Delta \vec{\alpha}$ \\
          $\vec{F} \leftarrow \vec{F} + \Delta\vec{\alpha}\vec{K}\vec{e(i)}$ \\
          \continue
        }
   }
   \tcp{Remove redundant support points}
   \If{$\max{\vec{y} (\vec{F} - \vec{\alpha})} > 0 \textup{ subject to } \vec{\alpha}\neq 0$}{
     $i \leftarrow \argmax \vec{y}(\vec{F} - \vec{\alpha})$ subject to $\vec{\alpha}\neq 0$\\
     $\vec{F} \leftarrow \vec{F} - \vec{\alpha}_i\vec{K}\vec{e(i)}$\\
     $\vec{\alpha}_i \leftarrow 0$ \\
     \continue
   }
   \break
 }


\tcp{Remove all elements corresponding to non-support points}
 $\alpha, \mathcal{S}, \vec{F}, \vec{K} \leftarrow \textup{removeNonsupportPoints}()$ \label{removeNonsupportPoints} \\

 \textbf{return} {$\vec{\alpha}$, $\mathcal{S}$}
 
 \caption{\strut Fastron Model Update Algorithm}
\label{alg:fastronUpdate}
\end{algorithm}

\begin{claim}
Minimization of $\mathcal{L}(\vec{\alpha})$ with the greedy coordinate descent rule defined in Eq. \ref{updateRule} and \ref{minMargin} will always eventually yield a model with positive margin for all samples given nonsingular kernel matrix $\vec K$.
\label{claim:convergence}
\end{claim}

\begin{proof}
If $\min_i{m^{(n)}(\mathcal{X}_i)}\leq 0$, an upper bound on the change in loss per descent step is $\sup\left(\mathcal{L}\left(\vec{\alpha}^{(n+1)}\right) - \mathcal{L}\left(\vec{\alpha}^{(n)}\right)\right)= \sup\left(-\frac{1}{2}(b_i-{m^{(n)}(\mathcal{X}_i)})^2\right)= -\frac{1}{2}$. A lower bound of $\mathcal{L}(\vec \alpha)$ is $\inf\mathcal{L}(\vec \alpha) = \mathcal{L}(\vec{K}^{-1}\vec{B}\vec{y}) = -\frac{1}{2}\vec{y}^\mathsf{T}\vec{B}\vec{K}^{-1}\vec{B}\vec{y}$ for nonsingular $\vec{K}$. A loose upper bound on the number of descent steps required to reach $\mathcal{L}(\vec{K}^{-1}\vec{B}\vec{y})$ from initial loss $\mathcal{L}(\vec{\alpha}^{(0)})$ is thus $\frac{\mathcal{L}(\vec{K}^{-1}\vec{B}\vec{y})-\mathcal{L}(\vec{\alpha}^{(0)})}{-\frac{1}{2}} = \vec{y}^\mathsf{T}\vec{B}\vec{K}^{-1}\vec{B}\vec{y} + 2\mathcal{L}(\vec{\alpha}^{(0)})$. The margin is $b_i$ for $\mathcal{X}_i$ when $\vec{\alpha} = \vec{K}^{-1}\vec{B}\vec{y}$, which is positive because $b_i \geq 1$. 

If $\min_i{m^{(n)}(\mathcal{X}_i)} > 0$, the model at iteration $n$ successfully provides a positive margin for all samples.
\end{proof}


Claim \ref{claim:convergence} means the weight update algorithm can terminate once all training samples have positive margin or will otherwise work toward achieving positive margin for all samples. Claim \ref{claim:convergence} is contigent on $\vec K$ being nonsingular, which is shown to be the case for the FK kernel in Claim \ref{claim:pdFKkernel}. 

After finding a model with positive margin for all training points, points in the training set with nonzero value in $\vec\alpha$ comprise the support set $\mathcal{S}\subseteq \mathcal{X}$. To promote model sparsity, support points are removed from $\mathcal{S}$ (by setting their corresponding weight to 0) if they would have positive margin even after their removal from $\mathcal{S}$.

The Fastron model update algorithm is summarized in Algorithm \ref{alg:fastronUpdate}. An iteration limit ${iter}_{max}$ and a limit on the size of the support set $\mathcal{S}_{max}$ can be set to limit update times and model sizes at the cost of model accuracy.  $\mathrm{loadPreviousModel}()$ initializes the values in each data structure to 0 or loads a previously-trained model if it exists, $\mathrm{computeKernelMatrixColumn}(i)$ computes the $i^{th}$ column of $\vec K$, and $\mathrm{removeNonsupportPoints}()$ removes elements from any data structure corresponding to a non-support point because these points are not needed for classification.

After model updating, a two-stage active learning strategy searches for changes in the C-space in case workspace obstacles move. In the first stage, random samples are generated near each support point to search for small perturbations to the C-space obstacles. In the second stage, uniformly random samples are generated in the C-space to search for new obstacles entering the workspace of the robot. A geometric collision check is performed for each support point and new sample generated via active learning to update $\vec y$. With the updated training set, the model update procedure is repeated.

\section{Results}
In this section, we evaluate the performance of the FK kernel by examining how it improves proxy collision checking and motion planning. We use the 7 DOF arms of the Baxter robot for these experiments. We use Fastron with the proposed FK kernel and with the original RBF kernel, denoted as Fastron FK and Fastron RBF, respectively. We select the parameters for Fastron FK and Fastron RBF with a grid search. The geometric collision detectors compared are GJK \cite{Gilbert1988} and FCL \cite{Pan2012}. GJK can determine whether two convex shapes are intersecting, and FCL is the collision detection library used in the MoveIt! motion planning framework \cite{Sucan}. We use GJK as a faster, approximate geometric method and use FCL as the ground truth collision method.

\subsection{Collision Checking}
\label{sec:colCheckSection}

\begin{table}[!b]
\footnotesize
\setlength{\tabcolsep}{6pt}
\centering
\ra{1.3}
\begin{tabularx}{\linewidth}{@{}Xlll@{}}
\toprule
Method & Query Time & $|\mathcal{S}|$ & Update Time \\
\midrule
Fastron FK& $2.5\pm 0.5\ \mu s$& $302.5\pm 109.6$& $55.8\pm 4.8\ ms$ \\
Fastron RBF& $4.6\pm 1.1\ \mu s$& $2368.9\pm 583.6$& $129.4\pm 30.2\ ms$ \\
GJK & $14.6\pm 0.2\ \mu s$& --- & --- \\
FCL & $22.9\pm 0.3\ \mu s$& --- & --- \\
\bottomrule
\end{tabularx}
\caption{Performance of Fastron FK and Fastron RBF compared against geometric collision detection methods (GJK and FCL). Lower is better for query time, model size $|\mathcal{S}|$, and model update time.}
\label{table:colCheckResults}
\end{table}

To compare collision checking performance, we compare model correctness, query times, model sizes, and model update times. Measures of model correctness include the overall classification accuracy, true positive rate (TPR), and true negative rate (TNR), defined as
\begin{align}
\mathrm{Accuracy} &= \frac{TP + TN}{TP + FN + TN + FP} \\
\mathrm{TPR} &= \frac{TP}{TP + FN} \\
\mathrm{TNR} &= \frac{TN}{TN + FP}
\end{align}
where $TP$ and $FP$ represent the number of samples in the test set that are correctly and incorrectly predicted to be in $\mathcal{C}_{obs}$, respectively, and $TN$ and $FN$ represent the same but for the $\mathcal{C}_{free}$ class. TPR and TNR are useful measures of correctness when the volumes of $\mathcal{C}_{obs}$ and $\mathcal{C}_{free}$ are unbalanced. Model size $|\mathcal{S}|$ is the count of support points required to represent the model for the proxy collision check methods. The results of the experiments are averaged over 50 trials, where each trial uses one randomly-sized obstacle that moves up to 30 times in a random direction, and the models are updated between each obstacle movement.

\begin{figure}[t]
	\includegraphics[width=\linewidth, trim={1.1cm 0.1cm 1.4cm 0.1cm}, clip]{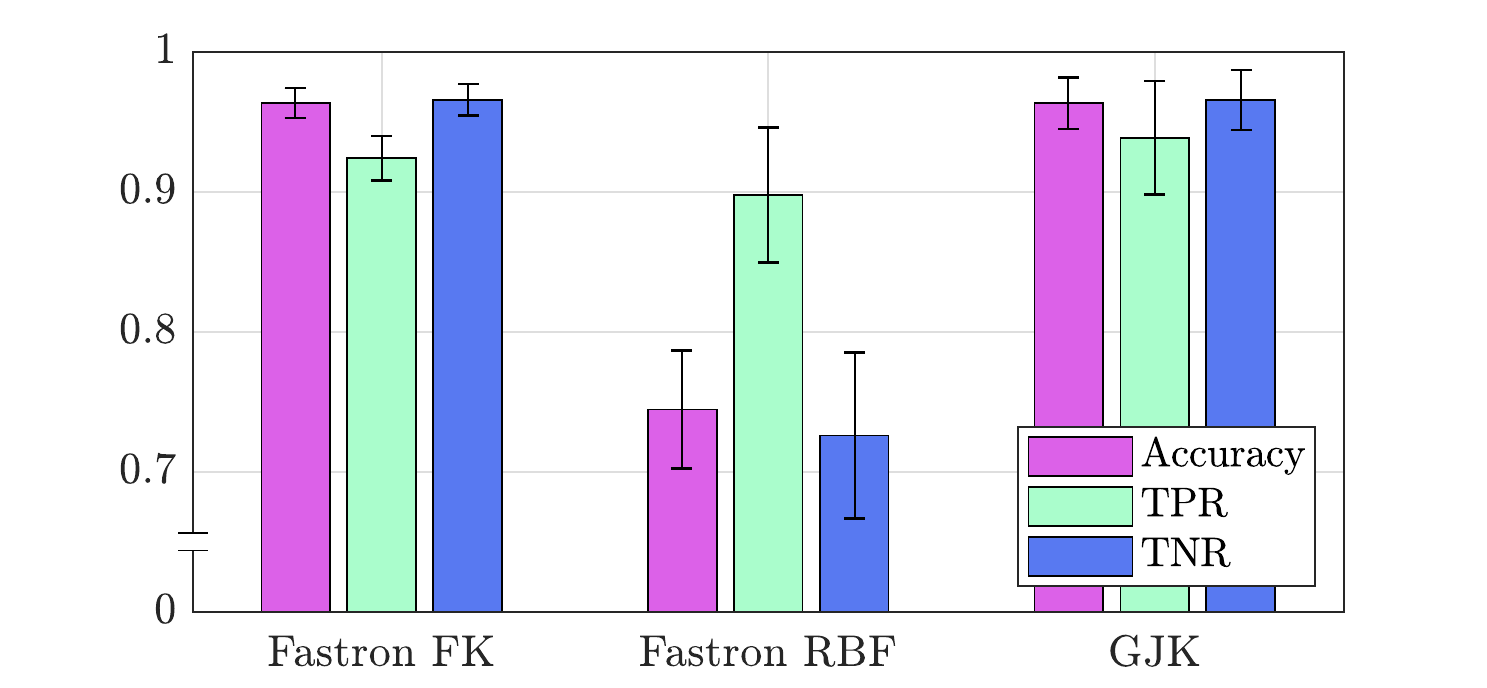}
	\caption{A comparison of the correctness of Fastron FK, Fastron RBF \cite{Das2019}, and GJK \cite{Gilbert1988} with a single, randomly-sized moving obstacle. Fastron FK has much better accuracy, TPR, and TNR than Fastron RBF. In fact, Fastron FK performs almost as well as GJK, a geometric collision checker.}
	\label{fig:accuracyPlot}
\end{figure}

\begin{figure}[b]
	\centering
    \subfloat[RBF Kernel] {
        \includegraphics[width=0.48\linewidth, trim={10cm 11cm 10cm 3cm}, clip]{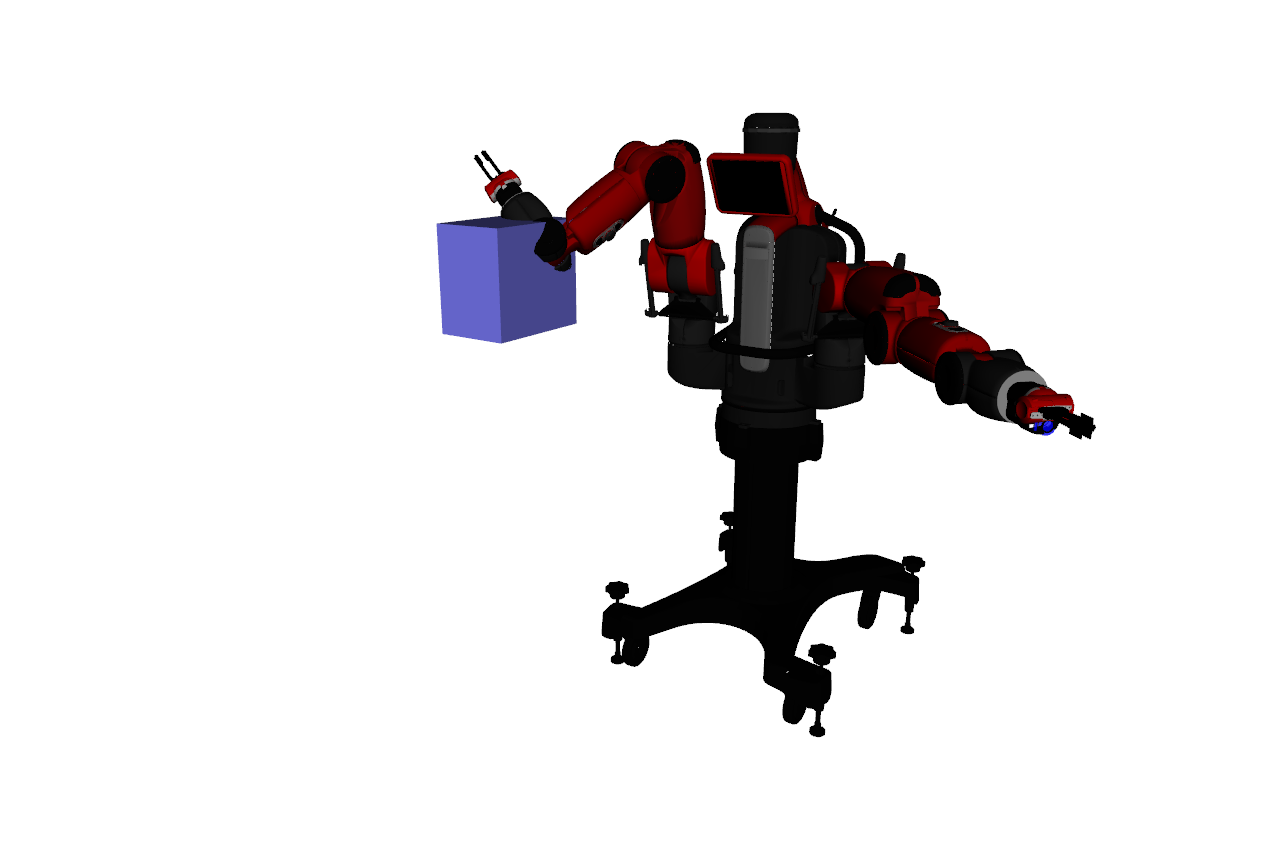}}
		\hfill
    \subfloat[FK Kernel] {
        \includegraphics[width=0.48\linewidth, trim={10cm 11cm 10cm 3cm}, clip]{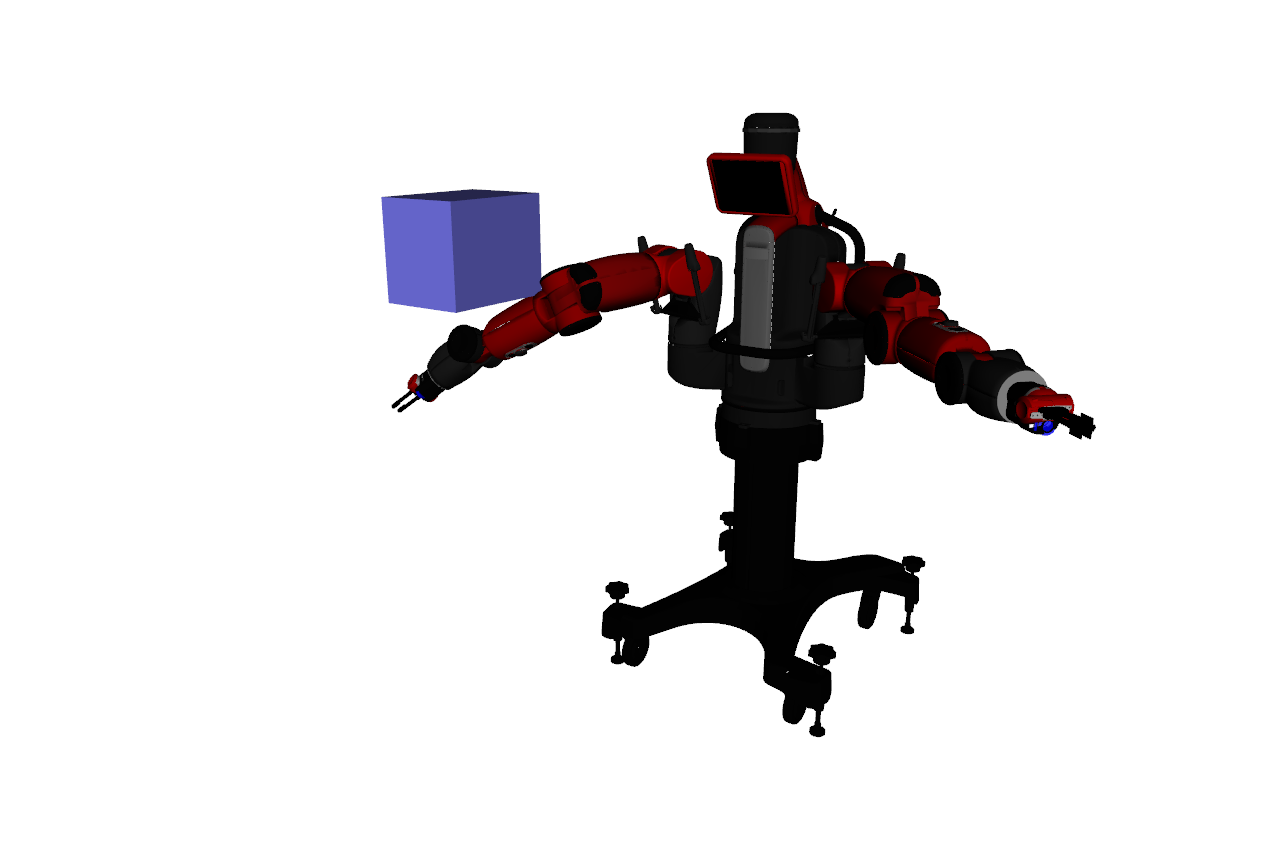}}
    \caption{Typical failure cases using the RBF kernel and the FK kernel, where the configuration should be labeled as $\mathcal{C}_{obs}$ but the Fastron model predicts $\mathcal{C}_{free}$. With the RBF kernel, the robot blatantly intersects with the workspace obstacle. With the FK kernel, the robot slightly intersects with the obstacle.}
    \label{fig:failureCases}
\end{figure}

\begin{figure*}[!hb]
	\centering
    \subfloat[RRT\label{fig:rrtstar}]{
        \includegraphics[width=0.32\linewidth, trim={0.6cm 0.4cm 1.3cm 0cm}, clip]{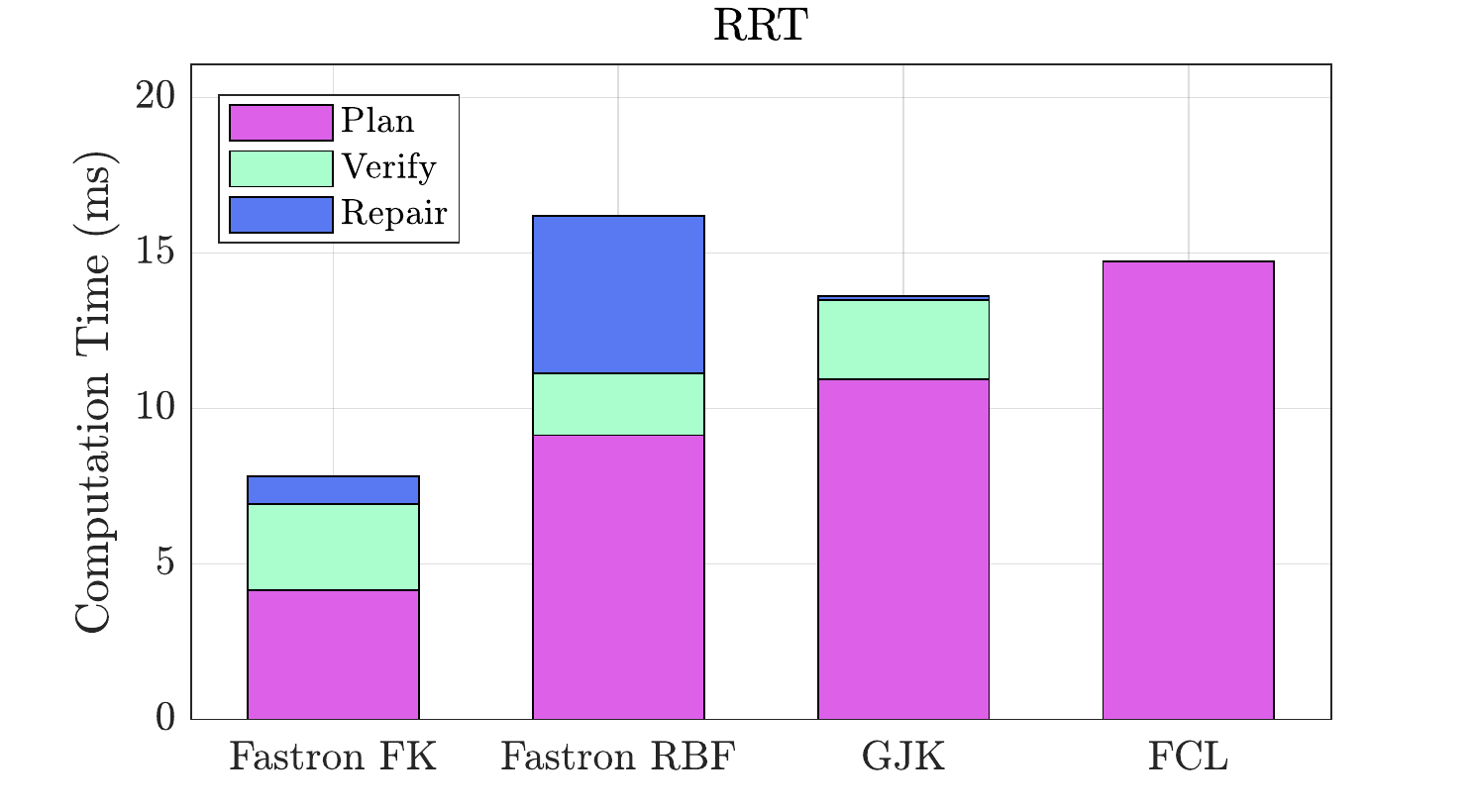}}
        \hfill
    \subfloat[RRT-Connect\label{fig:rrtconnect}] {
        \includegraphics[width=0.32\linewidth, trim={0.6cm 0.4cm 1.3cm 0cm}, clip]{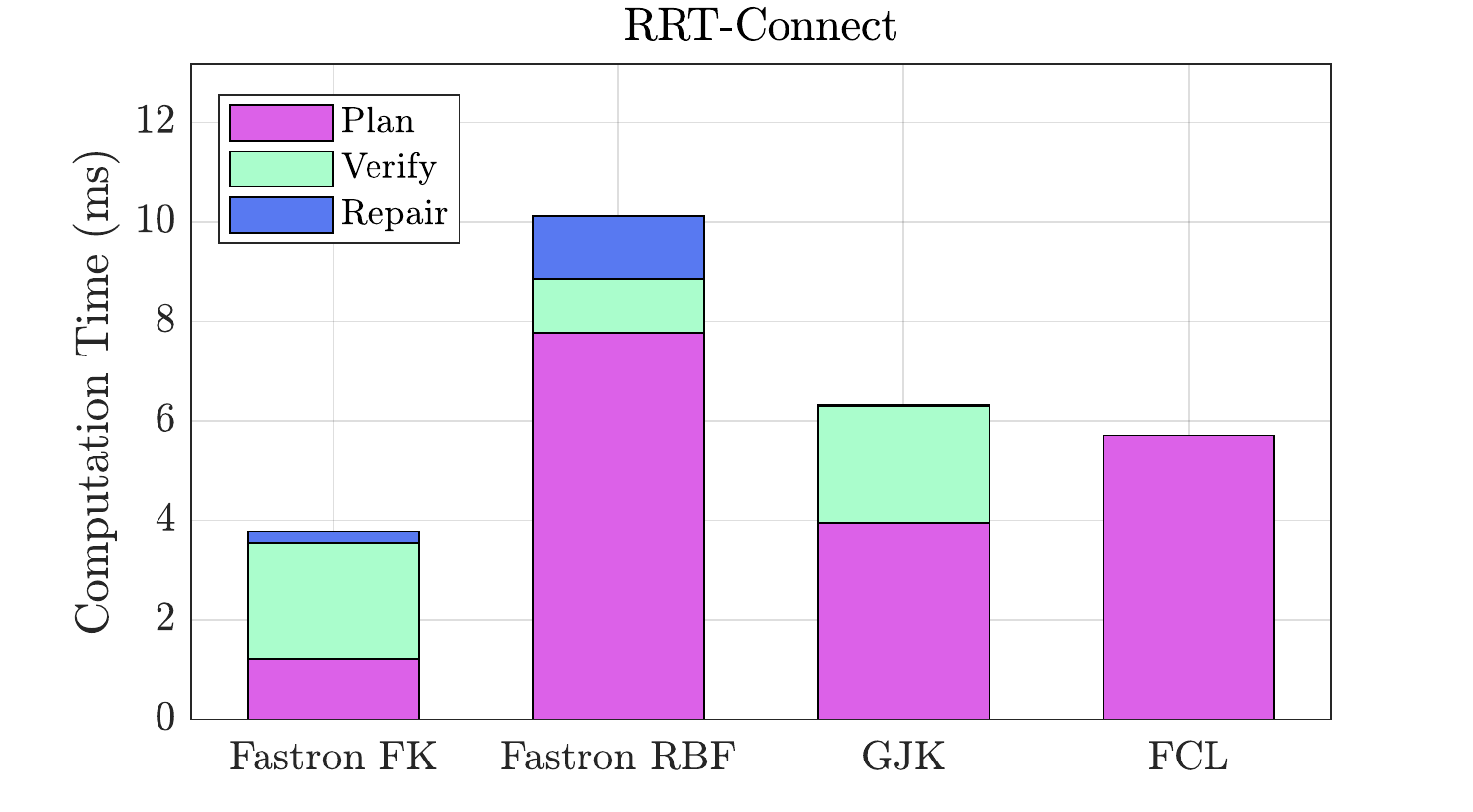}}
		\hfill
	\subfloat[SBL] {
        \includegraphics[width=0.32\linewidth, trim={0.6cm 0.4cm 1.3cm 0cm}, clip]{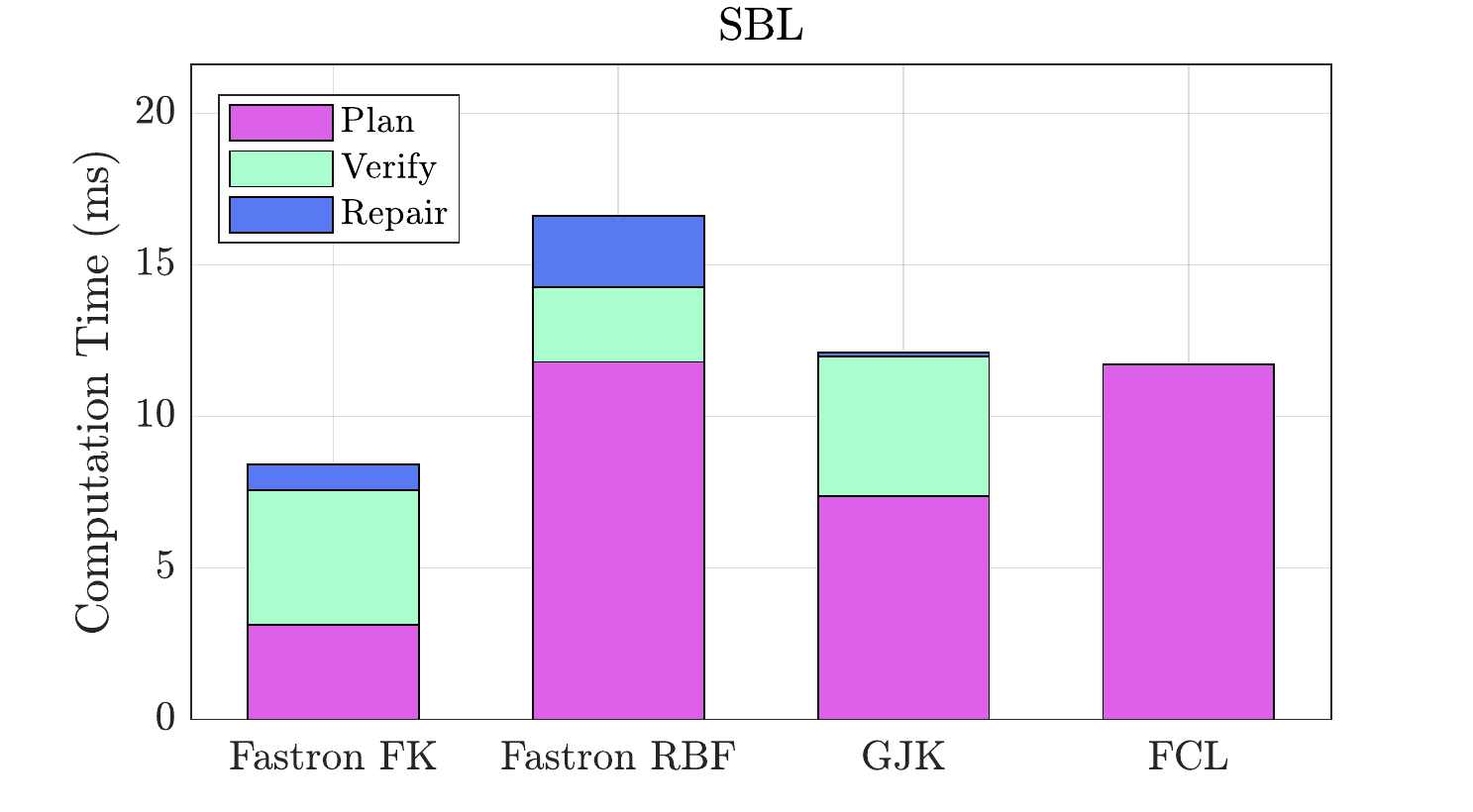}}
        \hfill
    \subfloat[RRT*] {
        \includegraphics[width=0.32\linewidth, trim={0.6cm 0.4cm 1.3cm 0cm}, clip]{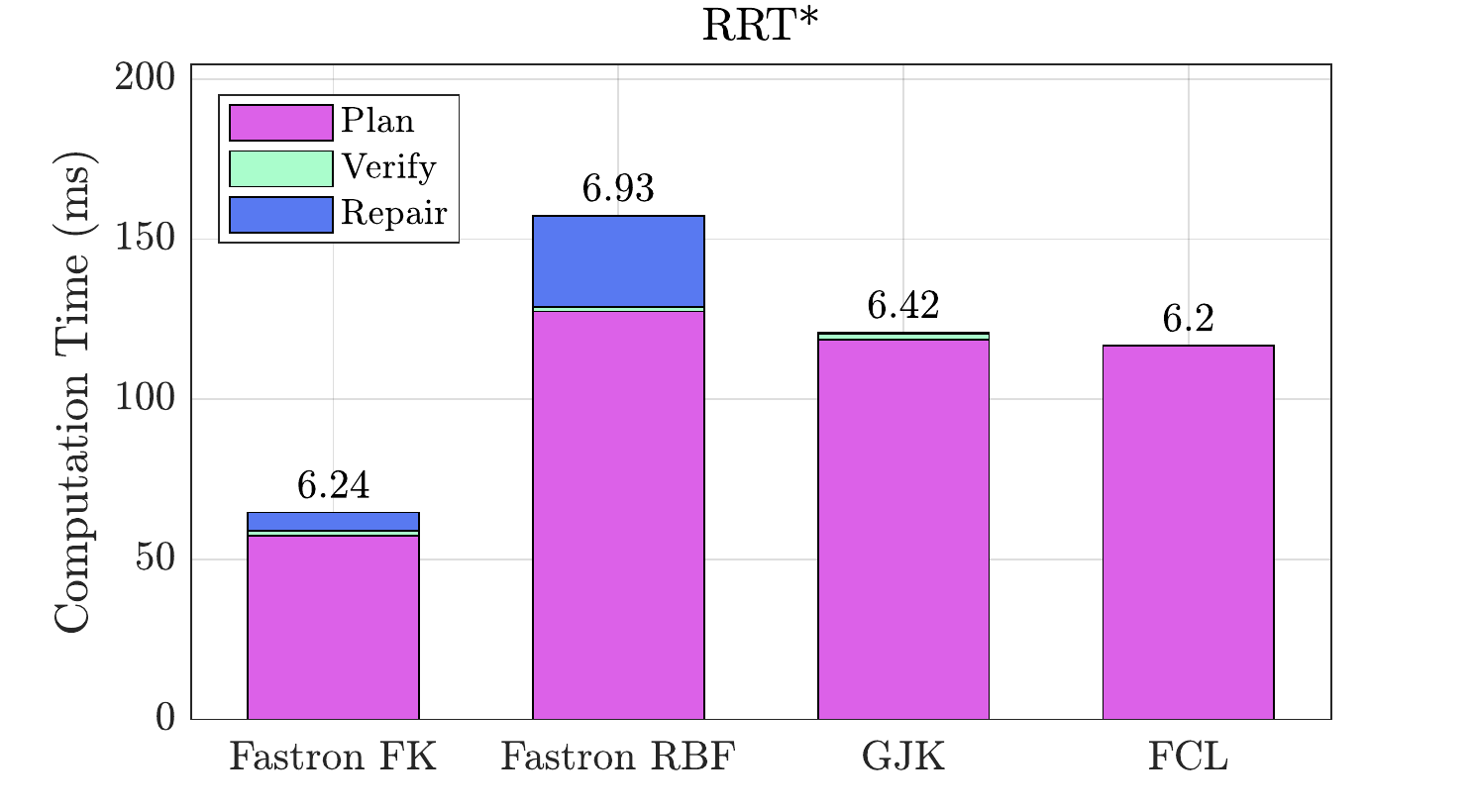}}
		\hfill
    \subfloat[FMT*] {
        \includegraphics[width=0.32\linewidth, trim={0.6cm 0.4cm 1.3cm 0cm}, clip]{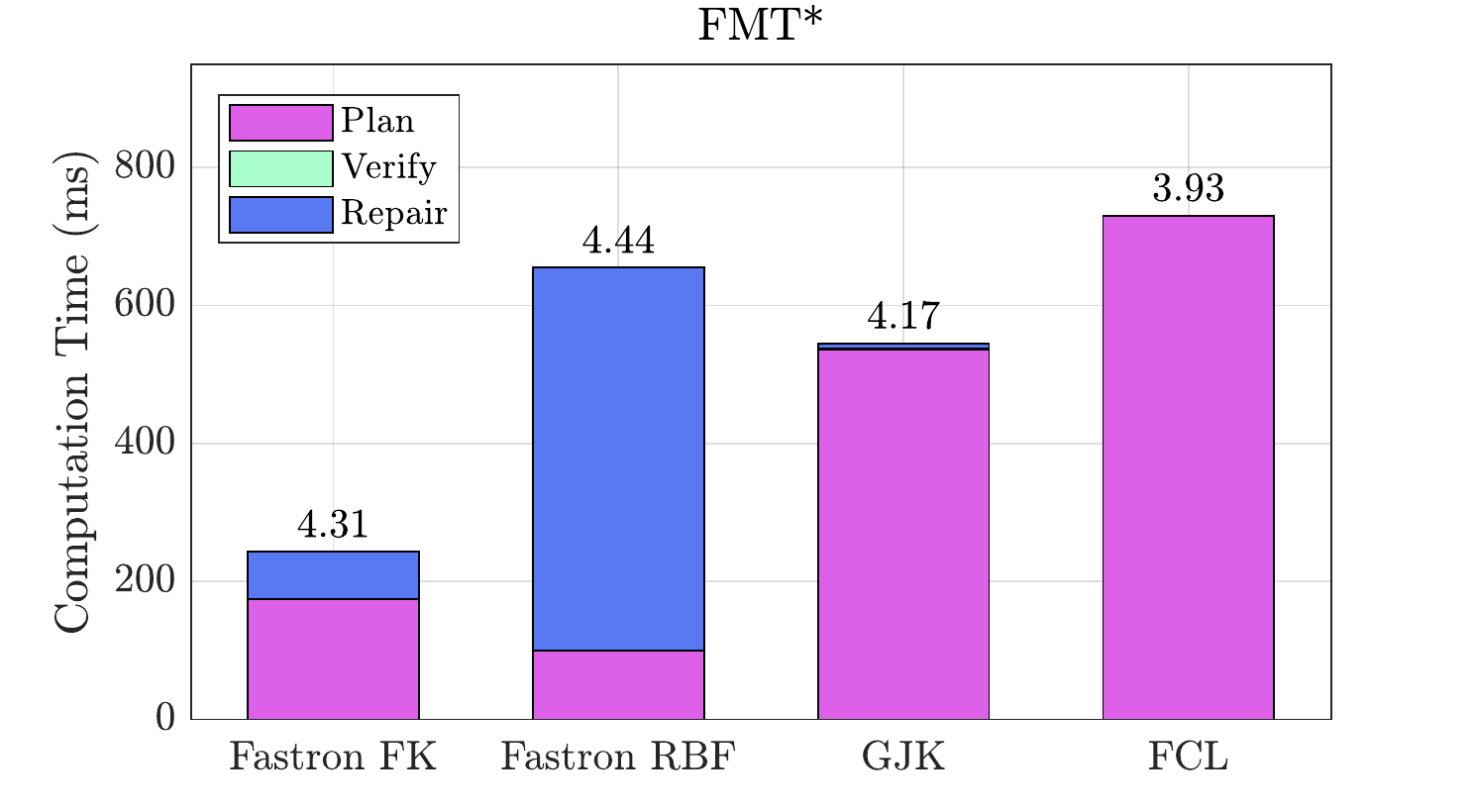}}       
		\hfill
	\subfloat[Informed RRT*] {
        \includegraphics[width=0.32\linewidth, trim={0.6cm 0.4cm 1.3cm 0cm}, clip]{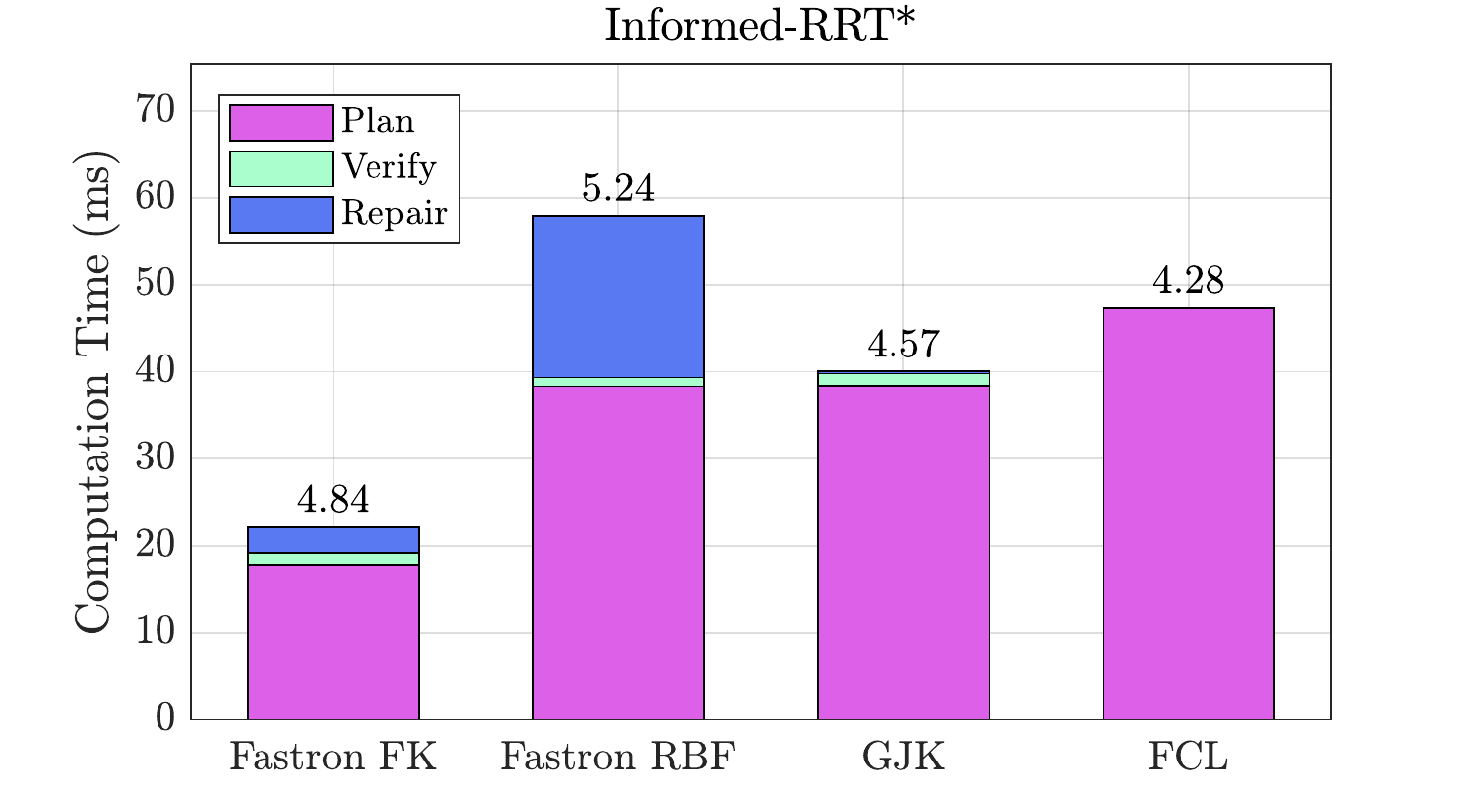}}

    \caption{Motion planning results using various collision detection methods: Fastron with FK kernel, Fastron with RBF kernel, GJK, and FCL (which serves as the ground truth collision detector). The obstacle in the workspace is repeatedly moved. For each approximated collision detection method, the path is validated and fixed using FCL in the Verify and Repair steps. The numbers  above each optimizing planner's bars are the average lengths of the initial motion plans.}
    \label{fig:motionPlanningChangingResults}
\end{figure*}

\begin{figure}[t!]
\includegraphics[width=\linewidth, trim={0cm 0cm 0.3cm 0cm}, clip]{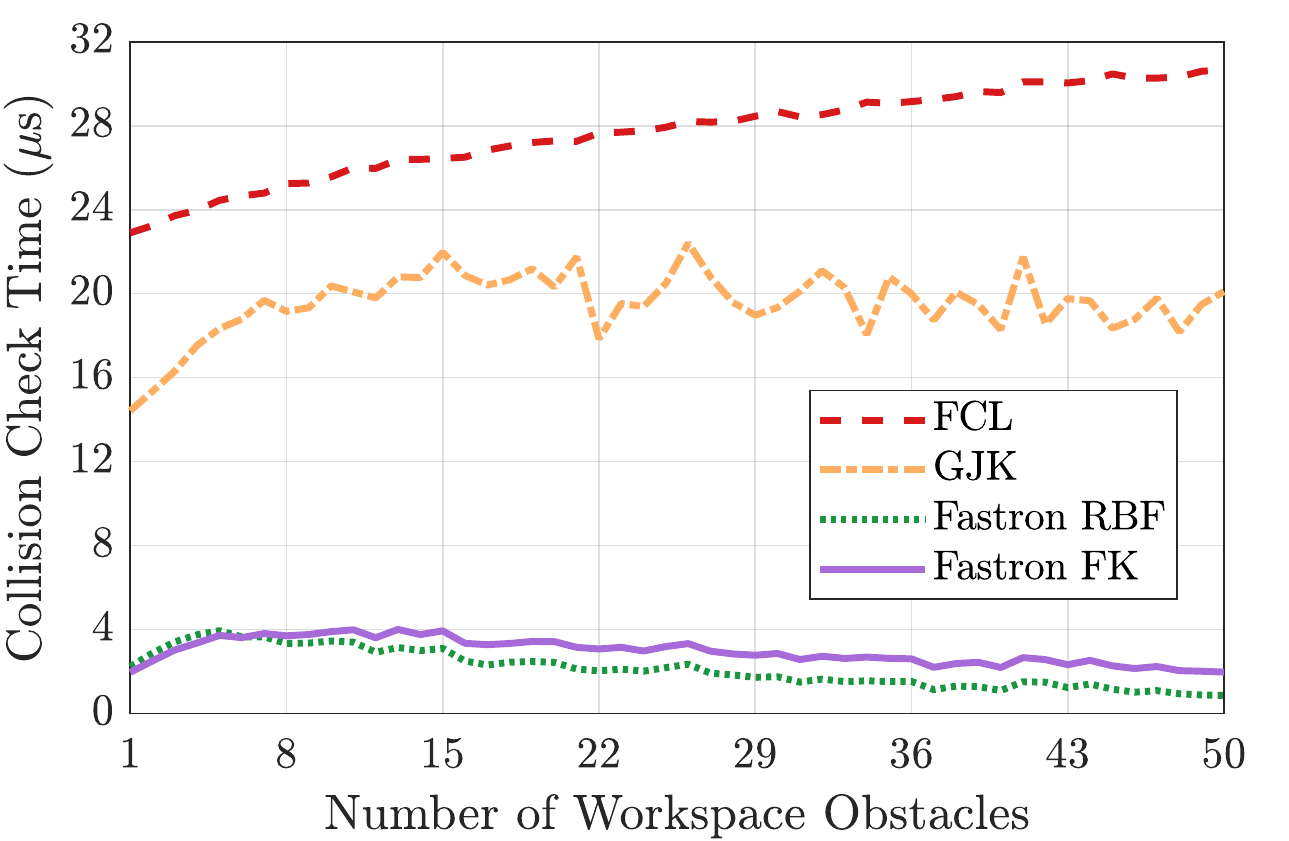}
\caption{Query times for Fastron FK, Fastron RBF \cite{Das2019}, GJK \cite{Gilbert1988}, and FCL \cite{Pan2012} against the number of workspace obstacles. Both GJK and FCL increase in query times as the number of obstacles increase, while both Fastron methods increase slightly before decreasing because the number of support points required to represent the space are the most when the class sizes are more balanced.}
\label{fig:numObsSweep}
\end{figure}

The query timings, model sizes, and model update timings are tabulated in Table \ref{table:colCheckResults}. The initial model training (i.e., labeling 5000 random configurations using FCL and training the model starting from $\vec\alpha = 0$) took on average 134.0 ms and 164.8 ms for Fastron FK and Fastron RBF, respectively. The majority of this initial training time (approximately 115 ms) for both methods is dedicated toward labeling the dataset using FCL. While Fastron FK has 7.8 times fewer support points than Fastron RBF, the FK kernel is more expensive to compute than the RBF kernel, causing Fastron FK to be only 1.8 times faster than Fastron RBF for query times. Fastron FK beats GJK and FCL by a factor of 5.8 and 9.2, respectively, for query times. Since there are fewer support points in the Fastron FK model, the update time (largely throttled by labeling new data via FCL) is about 2.3 times faster for Fastron FK than for Fastron RBF.

The correctness of the approximate collision methods are shown in Fig. \ref{fig:accuracyPlot}, illustrating the advantages of Fastron FK over Fastron RBF. The accuracy, TPR, and TNR of Fastron FK are much higher than those of Fastron RBF. In fact, Fastron FK performs similarly in terms of correctness to GJK, but is 5.8 times faster. Fastron FK achieves an accuracy of 96.4\%, a significant improvement over Fastron RBF's accuracy of 74.5\%. While the accuracy, TPR, and TNR are roughly balanced for Fastron FK, Fastron RBF's TPR is much higher than its accuracy and TNR because it prioritizes the $\mathcal{C}_{obs}$ label over the $\mathcal{C}_{free}$ label. While biasing toward the $\mathcal{C}_{obs}$ label yields a more conservative model, it also reduces the volume of the valid region of C-space and increases search time during sampling-based motion planning, which we will see is the case for Fastron RBF in the next section.

Fig. \ref{fig:failureCases} shows typical failure cases using Fastron FK and Fastron RBF, where the true collision status is $\mathcal{C}_{obs}$ but the Fastron models incorrectly predict $\mathcal{C}_{free}$. Using the RBF kernel, the robot often significantly intersects with the workspace obstacle, especially at its more distal links. Using the FK kernel, the robot usually slightly intersects with the workspace obstacles in its failure cases. A slight intersection with workspace obstacles is less severe than blatantly penetrating the obstacle, demonstrating the benefit of comparing control points in the FK kernel over comparing joint values directly with the RBF kernel.

In an additional experiment, we increase the number of obstacles to see the effect on query times. Fig. \ref{fig:numObsSweep} shows the query times for each method with respect to the number of workspace obstacles for up to 50 randomly-placed obstacles. We can see that the geometry-based methods, GJK and FCL, increase in query times as the number of obstacles increase. On the other hand, Fastron-based query times first increase and then decrease because fewer support points are required when one collision status is more prevalent than the other.


\subsection{Motion Planning}
We apply Fastron FK to a broad collection of the sampling-based motion planning algorithms implemented in OMPL \cite{Sucan2012}. We select RRT \cite{LaValle2008}, RRT-Connect \cite{Kuffner2000}, SBL \cite{Sanchez2003}, RRT* \cite{Karaman2011}, FMT* \cite{Starek2015}, and Informed-RRT* \cite{Gammell2014} to demonstrate each collision checking method's performance. RRT, its bidirectional variant RRT-Connect, and SBL are probabilistically complete motion planners \cite{Choset2005} that terminate once a path between start and goal is found. SBL includes lazy collision checking, which only calls the collision check routine when absolutely necessary. The other planners are optimizing planners that attempt to determine the shortest collision-free path between start and goal. Rather than letting the optimizing planners run until timeout (which would cause all planning times to be equal), we stop these planners once an initial feasible path has been found. While these paths may not be close to optimal, we include the average lengths of these initial paths in the plot of results.

We use similar environments as used in Section \ref{sec:colCheckSection}, but the obstacle is sized such that the arm must always move around it. The start and goal configurations are selected to be on opposite sides of the workspace obstacle. Each collision checker is applied to each motion planner for 50 trials, where a trial consists of up to 30 random movements of the workspace obstacle, and the motion planner is repeated 10 times for each obstacle position. When using the approximate collision checkers, we verify and repair (if necessary) each plan using FCL by replanning with FCL on the invalid part of the plan. The resulting motion plan is thus guaranteed to be collision-free according to FCL.

The average motion planning timings are shown in Fig. \ref{fig:motionPlanningChangingResults}. Even with the verify and repair steps, Fastron FK provides the fastest motion plans, providing up to 3 times faster collision-free motion plans compared to when using FCL. On the other hand, Fastron RBF often struggles to find a feasible path because it overestimates the $\mathcal{C}_{obs}$ subspace, causing its planning time to often be large or hit the time limit. Furthermore, the repair times for Fastron RBF are significantly larger than those for Fastron FK and GJK, illustrating that Fastron RBF often yields invalid paths.



\section{Conclusion}
In this paper we introduce the FK kernel, and we use it with the Fastron algorithm for proxy collision detection for manipulator arms. The FK kernel utilizes control points placed at each joint in a kinematic chain to better represent the relation between C-space and workspace. The FK kernel is positive definite, which guarantees that the Fastron algorithm will provide a model with positive margin for a dataset in a finite number of iterations.

Compared to the previously-used RBF kernel, the FK kernel required 8 times fewer support points to represent the configuration space model and allowed 2 times faster proxy collision detection. Prediction accuracy is significantly higher with the FK kernel (96\%) than with the RBF kernel (75\%). The FK kernel allowed proxy collision checking to be 9 times faster and motion planning to be up to 3 times faster than when using geometric methods.

Future work includes utilization of GPU parallelization for faster kernel evaluation and proxy collision checking and applications of the FK kernel in other kernel-based methods. 

\bibliographystyle{IEEEtran}
\bibliography{references}\end{document}